\documentclass[10pt,twocolumn,letterpaper]{article}

\usepackage[pagenumbers]{cvpr} % To force page numbers, e.g. for an arXiv version
\input{ushyphex}
\usepackage{graphicx}
\usepackage{amsmath}
\usepackage{amssymb}
\usepackage{booktabs}
\usepackage{bbm}
\usepackage{multirow}
\usepackage[english]{babel}
\usepackage{amsthm}
\usepackage{xcolor}
\usepackage{mdframed}
\usepackage[accsupp]{axessibility}
\usepackage{listings}

\definecolor{codegreen}{rgb}{0,0.6,0}
\definecolor{codegray}{rgb}{0.5,0.5,0.5}
\definecolor{codepurple}{rgb}{0.58,0,0.82}
\definecolor{backcolour}{rgb}{0.95,0.95,0.95}

\lstdefinestyle{mystyle}{
    backgroundcolor=\color{backcolour},   
    commentstyle=\color{codegreen},
    keywordstyle=\color{magenta},
    numberstyle=\tiny\color{codegray},
    stringstyle=\color{codepurple},
    basicstyle=\ttfamily\footnotesize,
    breakatwhitespace=false,         
    breaklines=true,                 
    captionpos=b,                    
    keepspaces=true,                 
    numbersep=5pt,                  
    showspaces=false,                
    showstringspaces=false,
    showtabs=false,                  
    tabsize=2
}

\usepackage[pagebackref,breaklinks,colorlinks]{hyperref}

\usepackage[capitalize]{cleveref}
\crefname{section}{Sec.}{Secs.}
\Crefname{section}{Section}{Sections}
\Crefname{table}{Table}{Tables}
\crefname{table}{Tab.}{Tabs.}

\usepackage{amsmath,amsfonts,bm}

\def\eqref#1{equation~\ref{#1}}
\def\1{\bm{1}}

\def\vd{{\bm{d}}}

\def\vf{{\bm{f}}}
\def\vg{{\bm{g}}}
\def\vh{{\bm{h}}}

\def\vp{{\bm{p}}}
\def\vq{{\bm{q}}}

\def\vx{{\bm{x}}}
\def\vy{{\bm{y}}}
\def\vz{{\bm{z}}}

\def\mH{{\bm{H}}}

\def\mX{{\bm{X}}}
\def\mY{{\bm{Y}}}
\def\mZ{{\bm{Z}}}

\DeclareMathAlphabet{\mathsfit}{\encodingdefault}{\sfdefault}{m}{sl}
\SetMathAlphabet{\mathsfit}{bold}{\encodingdefault}{\sfdefault}{bx}{n}

\newcommand{\softmax}{\mathrm{softmax}}

\def\cvprPaperID{10929} % *** Enter the CVPR Paper ID here
\def\confName{CVPR}
\def\confYear{2023}

\title{MaskCon: Masked Contrastive Learning for Coarse-Labelled Dataset}

\newtheorem{theorem}{\textsc{Theorem}}

\begin{document}

\author{Chen Feng, Ioannis Patras\\
Queen Mary University of London, UK\\
{\tt\small \{chen.feng, i.patras\}@qmul.ac.uk}
% \and
% Ioannis Patras\\
% Queen Mary, University of London\\
% London, UK\\
% {\tt\small i.patras@qmul.ac.uk}
}
\maketitle

% \twocolumn[{%
% \renewcommand\twocolumn[1][]{#1}%
% \maketitle
% \begin{center}
%     \centering
%     \captionsetup{type=figure}
%     \includegraphics[width=1\textwidth]{figures/caption2.png}
%  %   \caption{Toy visualization of our method. Considering a toy dataset with only 6 samples in total, different inter-sample relations will be generated according to different tbd}
%   \caption{Toy visualization of our method. Considering a toy dataset with only 6 samples in total, different inter-sample relations will be generated according to different tbd}
%     \label{fig:toyexample}
% \end{center}%
\begin{abstract}
Deep learning has achieved great success in recent years with the aid of advanced neural network structures and large-scale human-annotated datasets. However, it is often costly and difficult to accurately and efficiently annotate large-scale datasets, especially for some specialized domains where fine-grained labels are required. In this setting, coarse labels are much easier to acquire as they do not require expert knowledge. In this work, we propose a contrastive learning method, called \textbf{mask}ed \textbf{con}trastive learning~(\textbf{MaskCon}) to address the under-explored problem setting, where we learn with a coarse-labelled dataset in order to address a finer labelling problem. More specifically, within the contrastive learning framework, for each sample our method generates soft-labels  with the aid of coarse labels against other samples and another augmented view of the sample in question. 
By contrast to self-supervised contrastive learning where only the sample's augmentations are considered hard positives, and in supervised contrastive learning where only samples with the same coarse labels are considered hard positives, we propose soft labels based on sample distances, that are masked by the coarse labels. This allows us to utilize both inter-sample relations and coarse labels. 
We demonstrate that our method can obtain as special cases many existing state-of-the-art works and that it provides tighter bounds on the generalization error. 
Experimentally, our method achieves significant improvement over the current state-of-the-art in various datasets, including CIFAR10, CIFAR100, ImageNet-1K, Standford Online Products and Stanford Cars196 datasets. Code and annotations are available at \url{https://github.com/MrChenFeng/MaskCon_CVPR2023}.
\end{abstract}

\begin{figure}[htbp]
    \centering
    \includegraphics[width=1\linewidth]{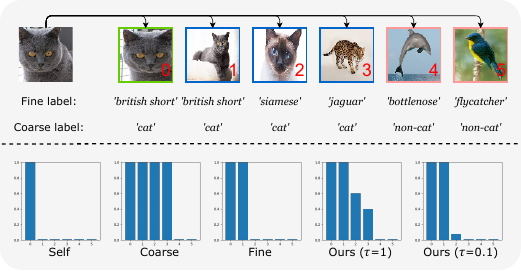}
\caption{Contrastive learning sample relations using MaskCon (ours) and other learning paradigms when only coarse labels are available. MaskCon are closer to the fine ones.}
    \label{fig:intro_mechanism_toy}
\end{figure}

\section{Introduction}
Supervised learning with deep neural networks has achieved great success in various computer vision tasks such as image classification, action detection and object localization. However, the success of supervised learning relies on large-scale and high-quality human-annotated datasets, whose annotations are time-consuming and labour-intensive to produce. To avoid such reliance, various learning frameworks have been proposed and investigated: Self-supervised learning aims to learn meaningful representations with heuristic proxy visual tasks, such as rotation prediction~\cite{gidaris2018rotpredict} and the more prevalent instance discrimination task, the latter, being widely applied in self-supervised contrastive learning framework; semi-supervised learning usually considers a dataset for which only a small part is annotated -- within this setting, pseudo labelling methods~\cite{pseudolabelling} and consistency regularization techniques~\cite{mixmatch, fixmatch} are typically used; Moreover, learning using more accessible but noisy data, such as web-crawled data, has also received increasing attention~\cite{ssr, dividemix}.

In this work, we consider an under-explored problem setting aiming at reducing the annotation effort -- learning fine-grained representations with a coarsely-labelled dataset. Specifically, we learn with a dataset that is fully labeled, albeit at a coarser granularity than we are interested in (i.e., that of the test set). Compared to fine-grained labels, coarse labels are often significantly easier to obtain, especially in some of the more specialized domains, such as the recognition and classification of medical pathology images. As a simple example, for the task of differentiation between different pets, we need a knowledgeable cat lover to distinguish between `British short' and `Siamese', but even a child annotator may help to discriminate between `cat' and `non-cat'~(\cref{fig:intro_mechanism_toy}). 
Unfortunately, learning with a coarse labelled dataset has been less investigated compared to other weakly supervised learning paradigms. 
% Fotakis et al.~\cite{coarse_theory} show theoretically that any problem based on fine-grained labels could be solved equally well as long as the coarse labels are sufficiently informative. However, direct learning with coarse labels will typically lead to under-clustering in the feature space, i.e., to a feature space where fine classes that share the same coarse label would have the same representations. 
Recently, Bukchi et al.~\cite{coarse_fewshot} investigate on learning with coarse labels in the few-shot setting. More closely related to us, Grafit~\cite{touvron2021grafit} proposes a multi-task framework by a weighted combination of self-supervised contrastive learning and supervised contrastive learning cost; Similarly, CoIns~\cite{xu2021coins} uses both a self-supervised contrastive learning cost and a supervised learning cross-entropy loss. Both works combine a fully supervised learning cost (cross entropy or contrastive) with a self-supervised contrastive loss -- these works are the main ones with which we compare. 

Differently than them, instead of using self-supervised contrastive learning as an auxiliary task, we propose a novel learning scheme, namely \textbf{Mask}ed \textbf{Con}trastive Learning~(\textbf{MaskCon}).  Our method aims to learn by considering inter-sample relations of each sample with other samples in the dataset. Specifically, we always consider the relation to oneself as confidently positive. To estimate the relations to other samples, we derive soft labels by contrasting an augmented view of the sample in question with other samples, and further improve it by utilizing the \textit{mask} generated based on the coarse labels.
% we implicitly rely on the idea of augmentation invariance in self-supervised contrastive learning by comparing an augmented view of the sample in question with other samples. 
Our approach generates soft inter-sample relations that can more accurately estimate fine inter-sample relations compared to the baseline methods~(\cref{fig:intro_mechanism_toy}). Efficiently and effectively, our method achieves significant improvements over the state-of-the-art in various datasets, including CIFARtoy, CIFAR100, ImageNet-1K and more challenging fine-grained datasets Stanford Online Products and Stanford Cars196. %~\footnote{The coarse labels for datasets except CIFAR100 that we have collected are publicly released with our code repository.}. 

\section{Related works}
In this section, we first briefly review the contrastive learning works which are the basis of our method and then, we briefly introduce the state-of-the-art works in related problem settings. 

%although learning with coarse labels is an under-explored problem~(),

\paragraph{Self-supervised contrastive learning}
\label{sec:contrastive_learning_rw}
Contrastive learning has recently emerged as a powerful method for representation learning without labels. In the context of self-supervision, contrastive learning has been used for instance discrimination, producing models whose performance rivals that of fully supervised training. The objective of contrastive instance discrimination is to minimize the distance between transformed views of the same sample in the feature space, while maximizing their distance from views of other samples. Notable methods in this area are SimCLR~\cite{simclr} and MoCo~\cite{moco}, as well as other variants including SwAV~\cite{swav}, which leverages clustering to identify prototypes for the contrastive learning objective and~\cite{wang2021solving}, that reformulates the instance discrimination contrastive objective in the form of a triplet loss. These methods have been found to produce models that, requiring only limited amounts of annotated data for fine-tuning on a particular task or domain, achieve performance that reaches or even surpasses that of supervised models~\cite{simclrv2, ericsson2021well}. Furthermore, self-supervised training has been found to compare positively with supervised training in other respects, such as producing models that perform better in the context of continual learning~\cite{gallardo2021self}.

However, self-supervised contrastive learning based on instance discrimination task usually suffers from the `false negative' problem -- that samples from the same class should not be considered as negative. For example, pictures of different cats should not be considered completely negative to each other. To this end, heuristically modifying the inter-sample relations have been widely applied in self-supervised learning~\cite{dwibedi2021nnclr, meanshift, CLID, ascl}, intuitively similar to us. Many of the state-of-the-art methods can be considered as special cases of our method by the adjustment of hyperparameters~(\cref{method:relation_to_otherworks}).

\paragraph{Combing contrastive learning with supervised learning}
\label{sec:combination_rw}
% Due to the great success of self-supervised contrastive learning, more and more attention has been paid to how to introduce contrastive learning into existing supervised learning paradigms.
% The direct adaptation of contrastive learning in the fully supervised setting is the supervised contrastive learning idea, which is first introduced in \cite{nca, cns}, named neighbourhood components analysis~(NCA) at that time. Supervised contrastive learning has recently re-emerged as a hot research topic, thanks to the remarkable progress achieved in \cite{supcon} by applying more advanced network structures and image augmentations. Compared to common supervised learning based on cross-entropy loss, supervised contrastive learning has superior performance in hard sample mining~\cite{supcon}. 
The success of self-supervised contrastive learning has led to increasing attention on how to integrate contrastive learning into existing supervised learning paradigms. One approach is supervised contrastive learning, which adapts contrastive learning to the fully supervised setting. Supervised contrastive learning was first introduced in \cite{nca, cns} under the name of neighbourhood components analysis (NCA). Recently, supervised contrastive learning has regained interest due to the significant progress made in \cite{supcon} by employing more advanced network structures and image augmentations. Compared to standard supervised learning using cross-entropy loss, supervised contrastive learning exhibits superior performance in hard sample mining~\cite{supcon}.

Other works have attempted to introduce self-supervised contrastive learning as an auxiliary task, especially for weakly supervised learning problems. Gidaris et al.~\cite{gidaris2019boosting} jointly train a supervised classifier and a rotation-prediction objective. Komodakis and Gidaris~\cite{komodakis2018unsupervised} seeking to improve model performance on few-shot learning. Wei et al.~\cite{wei2020can} propose a label-filtered instance discrimination objective for the pre-training of models to be used for transfer learning in other datasets. S4~\cite{zhai2019s4l} uses rotation and exemplar~\cite{dosovitskiy2014discriminative} prediction to improve performance in semi-supervised learning tasks. Feng et al.~\cite{ssr} apply a self-supervised contrastive loss to safely learn meaningful representations in the presence of noisy labels.

\paragraph{Hierarchical image classification \& Unsupervised clustering}
Hierarchical image classification methods aim to learn better representations with datasets having hierarchically structured labels. To utilize information from different levels of labels,  Zhang et al.~\cite{zhang2022hierarchical_alllabels} propose a multi-losses contrastive learning framework with each loss considering a specific level of labels.
% \paragraph{Unsupervised clustering}
Unsupervised clustering methods often have the same problem setting but different goals than self-supervised learning, that is, to recover and identify the ground-truth semantic labels of each sample~\cite{caron2018deepclustering, van2020scan}. DeepCluster~\cite{caron2018deepclustering} simultaneously learns the clustering assignments and a parametric classifier, by feeding the clustering assignments as pseudo labels to the classifier. SCAN~\cite{van2020scan} inherits the idea of DeepClustering while augmenting extra self-labelling and neighbor mining techniques.

\section{Methodology}\label{method}
%When the labels of the training data and the labels of the test data are inconsistent in granularity, and in particular when the labels of the training set are coarser, we have to learn fine-grained information with such coarse labels. We first set up the notations essential for our method.
We consider the problem of learning when the labels of the training data and the labels of the test data are inconsistent in granularity, and in particular when the labels of the training set are coarser. More specifically, let us denote with $\mX = \{\vx_i\}_{i=1}^N$ an i.i.d sampled train dataset with annotated \textit{coarse labels} $\mY = \{\vy_i \in \{0, 1\}^M\}_{i=1}^N$ subject to $\sum_{m=1}^M \vy_{im} = 1$. Here, $N$ denotes the number of samples and $M$ denotes the number of coarse classes. Let us also define a finely labelled set $\mY' = \{\vy_i' \in \{0, 1\}^{M'}\}_{i=1}^N$ subject to $\sum_{k=1}^{M'} \vy_i^k = 1$, where $M'$ denotes the number of fine classes. Typically, $M' > M$. 

\subsection{Preliminaries}\label{method_pre}
We first quickly review the common learning paradigms utilized in the baseline methods and our method.
\subsubsection{Supervised learning}\label{method_pre_supervised}
In most supervised learning frameworks, the objective is the minimization of the below empirical risk:
\begin{equation}\label{eq:cerisk}
    R(f, g) = \sum_{i=1}^N L_{ce}(\vx_i, \vy_i; f, g),
\end{equation}
where $f$ denotes the feature encoder, and $g$ denotes the classifier head~(usually a single fully-connected layer).

For brevity, we denote $\vf_i \triangleq f(\vx_i)$, $\vg_i \triangleq g(f(\vx_i)$ and $\vp_i \triangleq \softmax(\vg_i)$ as the \textit{feature}, \textit{logit} and \textit{prediction} of a sample view $\vx_i$ respectively. Here,
\begin{equation}\label{eq:celoss}
    L_{ce}(\vx_i, \vy_i; f, g) = -\sum_{m=1}^M \vy_i^m\log \vp_i^m
\end{equation}
is the cross-entropy loss function -- the default loss function for classification problems.

\subsubsection{Contrastive learning}
Unlike the common supervised learning model above where a parametric classifier $g$ is learned based on the semantic labels $\mY$ of the samples $\mX$, we can also perform contrastive learning based on the inter-sample relations $\mZ = \{\vz_i \in (0, 1)^N\}_{i=1}^N$, with each entry $z_{ij}$ depicting the inter-sample relation between $\vx_i$ and $\vx_j$. 
Intuitively, $z_{ij} = 1$ means that sample $\vx_i$ and $\vx_j$ generate a strong positive pair. Since each sample may form multiple positive sample pairs, for brevity, we abuse the notation here with $\mZ$ denoting also the sample-wise normalized inter-sample relations.
To learn such inter-sample relations, instead of a parametric classifier $g$, the $f$ encoder is usually followed by a projector $h$, which is often implemented as an MLP and learned by regularizing the inter-sample relations $\mZ$~(\cref{eq:selfconrisk}). 

More specifically, let us denote by $\vh_i \triangleq h(f(\vx_i)$ the \textit{projection}. We first calculate the cosine similarity $\vd_i$ between a sample $\vx_i$ and the dataset $\mH = \{\vh_n\}_{n=1}^N$:
\begin{equation}\label{eq:cosinedist}
    \vd_i = [\cos(\vh_i, \vh_1), \cos(\vh_i, \vh_2),..., \cos(\vh_i, \vh_N)],
\end{equation}
Let us further define $\vq_i \triangleq \softmax(\vd_i/\tau_0)$, where $\tau_0$ is the temperature hyperparameter. Then the following empirical risk will be optimized:
\begin{equation}\label{eq:selfconrisk}
    R(f, h) = \sum_{i=1}^N L_{con}(\vx_i, \vz_i; f, h),
\end{equation}
where the contrastive loss $L_{con}$ is defined as follows:
\begin{equation}\label{eq:conloss}
    L_{con}(\vx_i, \vz_i; f, h) =  -\sum_{n=1}^N \vz_i^n\log \vq_i^n.
\end{equation}

% For contrastive learning, we apply the MoCo structure as the backbone.

\noindent\textbf{Self-supervised contrastive learning}
\label{method_pre_selfcontratsive}
We first introduce the most prevalent form of contrastive learning currently --- self-supervised contrastive learning for learning without labels. Since there are no annotated labels, we usually set $\mZ^{self}$ by considering each sample having only itself as positive:
\begin{equation}
    z^{self}_{ij}=\left\{
                \begin{array}{ll}
                  1, &\text{if}\ i = j\\
                   0, &\text{if}\ i \neq j
                \end{array}
              \right.
\end{equation}
Typically, we aim at maximizing the relations between different augmented views of the same sample while minimizing the relations between augmented views of different samples. This is also widely known as instance discrimination task. We denote such self-supervised contrastive loss as $L_{selfcon}$.

% \vspace{0.2cm}
\noindent\textbf{Supervised contrastive learning}
\label{method_pre_supervisedcontrastivelearning}
We can easily transition from self-supervised contrastive learning to supervised contrastive learning, by simply changing the inter-sample relations $\mZ^{sup}$ for \cref{eq:selfconrisk} and \cref{eq:conloss}. More specifically, as in \cite{nca, cns, supcon}, samples from the same semantic classes are considered positive pairs, 
\begin{equation}
    z^{sup}_{ij}=\left\{
                \begin{array}{ll}
                  1, &\text{if}\ \vy_i = \vy_j\\
                   0, &\text{if}\ \vy_i \neq \vy_j
                \end{array}
              \right.
\end{equation}
in addition to the sample itself (and its augmented views) as in self-supervised contrastive learning. We denote the supervised contrastive loss as $L_{supcon}$.

\begin{mdframed}[linecolor=lightgray!20,backgroundcolor=lightgray!20,everyline=true,linewidth=2pt,roundcorner=0pt,innertopmargin=0pt,
    innerbottommargin=0pt,
    innerrightmargin=0pt,
    innerleftmargin=0pt,
        leftmargin = 0pt,
        rightmargin = 0pt
        ]%%<---addendum
\textit{\textbf{Memory bank} For consistency with the supervised cross-entropy loss and for clarity, we present all the essential formulas above in terms of the entire dataset $\mX$. However, due to computation and memory constraints, it is often unrealistic to contrast a specific sample $\vx_i$ with the whole dataset.  In the actual implementation in this work, following MoCo\cite{moco}, we use a dynamic \textit{FIFO} memory bank $\mH = \{\vh_p\}_{p=1}^P$ consisting of cached projections of $P$ samples. For each sample $\vx_i$, we generate two random augmented views as $\vx_q$~(query) and $\vx_k$~(key) with their corresponding projections denoted by $\vh_q$ and $\vh_k$, respectively. Then, we calculate the cosine similarity between $\vh_q$ on the one hand, and $\vh_k$ and each projection in the memory bank on the other: 
\begin{equation}\label{eq:mocosim}
    \vd_i = [\overbrace{\cos(\vh_q, \vh_k)}^{self}, \overbrace{\cos(\vh_q, \vh_1),..., \cos(\vh_q, \vh_P)}^{memory \ bank}],
\end{equation}
For more details, please refer to the MoCo paper~\cite{moco}.}
\end{mdframed}

\subsubsection{Baseline methods}
Clearly, with respect to the finer target question labels, each sample $\vx_i$, is a trustworthy positive sample in relation to itself (i.e., they have the same fine-grained label). However, considering the samples with the same coarse label as positives may lead to under-clustering issues~(\cref{fig:tsne}) (they may have different fine-grained labels). 
% To cater for such under-clustering, recent works either try to combine extra instance discrimination task such as Grafit~\cite{touvron2021grafit}, or try to combine the supervised contrastive loss with self-supervised contrastive learning such as CoIns~\cite{xu2021coins}. 
% To cater for such under-clustering, recent works try to combine an extra self-supervised contrastive loss with the specific supervised loss utilizing coarse labels, such as supervised contrastive loss in Grafit~\cite{touvron2021grafit}, or the supervised cross-entropy loss in CoIns~\cite{xu2021coins}. 
% Therefore, the under-clustering tendency can be mitigated to some extent by the over-clustering effect of the self-supervised contrastive loss.
Recent works attempt to address such under-clustering by combining an additional self-supervised contrastive loss with a specific supervised loss that utilizes coarse labels. Grafit~\cite{touvron2021grafit} combine it with the supervised contrastive loss and CoIns~\cite{xu2021coins} with the the supervised cross-entropy loss. 
%, such as the supervised contrastive loss in Grafit~\cite{touvron2021grafit} or the supervised cross-entropy loss in CoIns~\cite{xu2021coins}. 
In both cases, as a result, the self-supervised contrastive loss that considers each sample as a class by itself, aims to mitigate the tendency for under-clustering. 
Formally, given our definitions of the previously mentioned loss functions in \cref{method_pre} ($L_{ce}$, $L_{selfcon}$ and $L_{supcon}$) it turns out that the losses used by Grafit~\cite{touvron2021grafit} and CoIns~\cite{xu2021coins} can be expressed as follows: 
\begin{align}
    L_{Grafit} &= wL_{supcon} + (1-w)L_{selfcon} \\
    L_{CoIns} &= wL_{ce} + (1-w)L_{selfcon}
\end{align}
where $w$ controls the relative weight of each loss.

\subsection{MaskCon: Masked Contrastive learning}\label{method_maskcon}
%Instead of equally weighing the samples from the same coarse classes, we aim to emphasize the `real' positive samples, i.e, samples with the same fine labels and ignore the others. To do so, we propose a novel contrastive learning framework, namely \textbf{Mask}ed \textbf{Con}trastive learning~(\textbf{MaskCon})~\footnote{The $\softmax$-based classification head technically can only handle close-world classification well, we thus build our method on the basis of contrastive learning as it focuses on inter-sample relations directly. }. 
Instead of equally weighing all the samples within the same coarse classes, we aim to emphasize the samples with the same fine labels and reduce the importance of the others. To achieve this, we introduce a novel contrastive learning method, namely \textbf{Mask}ed \textbf{Con}trastive learning~(\textbf{MaskCon}), within the framework of contrastive learning that utilizes inter-sample relations directly.
% To do so, within the framework of contrastive learning that utilizes inter-sample relations directly, we propose a novel contrastive learning method, namely \textbf{Mask}ed \textbf{Con}trastive learning~(\textbf{MaskCon}).

More specifically, for sample $\vx_i$, we estimate its inter-sample relations $\vz'_i$ to other samples utilizing the key view projection $\vh_k$ and the whole dataset $\{\vh_1, ..., \vh_N\}$ excluding itself~(since it will always be considered as a trustworthy positive), as below:
\begin{equation}\label{eq:mask_inter}
    z'_{ij} =  \frac{\mathbbm{1}(\vy_j = \vy_i)\cdot \exp(d'_{ij}/\tau)}{\sum_{n=1, n\neq i}^N\mathbbm{1}(\vy_n = \vy_i)\cdot \exp(d'_{in}/\tau)}, i\neq j,
\end{equation}
where the similarity $\vd'_i$ is given by
\begin{align}\label{eq:pseudosim}
\begin{split}
\vd'_i = &[\cos(\vh^k_i, \vh_1),...,\cos(\vh^k_i, \vh_{i-1}),\\&\cos(\vh^k_i, \vh_{i+1}),..., \cos(\vh^k_i, \vh_N)].
\end{split}
\end{align}
% Please note the difference between \cref{eq:cosinedist} and \cref{eq:pseudosim}, as we only estimate the inter-sample relations to other samples excluding itself. 
% as we only consider the cached samples in the memory bank for estimating the inter-sample relations to other samples excluding itself.
Please note the use of the mask~($\mathbbm{1}(\vy_j = \vy_i)$) that excludes from the $\softmax$ that estimates inter-sample relationships, the samples $j$ that have a different coarse label with the sample $i$ (and sets their $z'_{ij}$ to 0). While it is risky to consider all samples from the same coarse class as positive, we can confidently identify those samples that do not have the same coarse class as negative. This reduces the noise in $\vz'_i$.
%The mask~($\mathbbm{1}(\vy_j = \vy_i)$) helps to exclude these confident negative samples in estimating and reducing the noise in $\vz'_i$. 
Finally, we re-scale the $\vz'_i$ with its maximum
\begin{equation}
    z'_{ij} = z'_{ij} / \max(\vz'_i),
\end{equation}
to make the closest neighbour as positive as the sample itself and arrive at:
\begin{equation}
    z^{mask}_{ij}=\left\{
                \begin{array}{ll}
                  1, &\text{if}\ i=j\\
                  z'_{ij}, &\text{if}\ i\neq j
                \end{array}
              \right.
\end{equation}
Compared to $\mZ^{supcon}$, we thus reweight the samples of the same coarse label according to the similarities in the feature space.

We denote the masked contrastive loss as $L_{maskcon}$ and, similarly to Grafit and CoIns, we also consider a weighted loss as the final objective:
\begin{align}
    L &= wL_{maskcon} + (1-w)L_{selfcon}
\end{align}

\paragraph{Relations to SOTA works}\label{method:relation_to_otherworks}
By adjusting $w$ and $\tau$, our method can obtain various existing SOTA methods as special cases. More specifically, by setting $\tau$ to $\infty$, our method degenerates to Grafit~\cite{touvron2021grafit}. Ignoring the mask, i.e., treating all samples as having the same coarse label,
%~\footnote{An unlabelled dataset can also be considered as if all samples belong to the same coarse class, thus we ignore the mask here since it equals to one everywhere.}, 
our method can also generalize to existing SOTAs in self-supervised learning. For example, NNCLR~\cite{dwibedi2021nnclr} instead of considering only each sample itself as positive it considers as positive the nearest neighbors as well. Formally, by setting $w = 1$ and $\tau = 0$, MaskCon degenerates to NNCLR; ASCL~\cite{ascl} proposes to introduce extra positive samples by calculating soft inter-sample relations, adaptively weighted by its own normalized entropy. By setting $w=\mathsf{norm\_entropy}(\vz'_i)$, MaskCon degenerates to ASCL. 

\subsection{Theoretical justification of MaskCon}
% The core of contrastive learning is to find $\mZ$ that is suitable for the problem granularity and accurately reflects the relationship between samples. Supposing there exists such a hidden optimal $\hat{\mZ}$, similar to the empirical risk in~\cref{eq:selfconrisk}, we have the below expected population risk:
The fundamental goal of contrastive learning is to identify $\mZ$ that suits the problem granularity and accurately represents the relationship between samples. Assuming that there exists an optimal hidden $\hat{\mZ}$, we can define the expected population risk in terms of $\hat{\mZ}$ as shown in~\cref{eq:populationrisk}. 
\begin{equation}\label{eq:populationrisk}
    R_E(f, h) = E_{\vx_i} [L_{con}(\vx_i, \hat{\vz_i}; f, h)]
\end{equation}
Moreover, let us define the empirical risk with respect to $\hat{\mZ}$ as shown in~\cref{eq:optimalrisk}.
% Let us also define the empirical risk w.r.t $\hat{\mZ}$:
\begin{equation}\label{eq:optimalrisk}
      \hat{R}(f, h) = \sum_{i=1}^N L_{con}(\vx_i, \hat{\vz_i}; f, h)
\end{equation}

\begin{theorem}\label{theorem1}
% The upper and lower generalization error bound of contrastive learning depends on the difference between the inter-sample relations $\mZ$ and the optimal hidden inter-sample relations $\hat{\mZ}$: 
The generalization error bounds of contrastive learning depends on the difference between the inter-sample relations $\mZ$ and the optimal hidden inter-sample relations $\hat{\mZ}$, i.e.,
\begin{equation}
\label{eq:zdiff}
    d_z = \frac{1}{n}\sum_{i=1}^n \lVert \vz_i - \hat{\vz_i} \rVert_2
\end{equation}
\end{theorem}
\begin{proof}
Let us assume $L_{con}(\vx_i, \vz_i; f, h) \in [a, b]$ and is $\lambda$-Lipschitz continuous w.r.t $\vz_i$. Let $|\mathcal{F}|$ and $|\mathcal{H}|$ be the covering numbers that correspond to the finite hypothesis space of $f$ and $g$.
We have:
\begin{align*}
    |R_E(f, h) - R(f, h)| &\leq  \overbrace{|R_E(f, h) - \hat{R}(f, h)|}^{\mathsf{Hoeffding’s \ inequality}} \\
    &+  \overbrace{|\hat{R}(f, h) - R(f, h)|}^{\mathsf{Lipschitz \ continuous}} \\
    &\leq |a-b|\sqrt{\frac{log(2|\mathcal{F}|\cdot|\mathcal{H}|/\delta)}{2n}} \\
    &+ \lambda \frac{1}{n}\sum_{i=1}^n \lVert \vz_i - \hat{\vz_i} \rVert_2
\end{align*}
with probability at least $1-\delta$.
\end{proof}

Based on \cref{theorem1}, it can be concluded that our method performs better than Grafit with high probability, because an appropriate temperature $\tau$ allows $\mZ^{mask}$ to more accurately estimate the optimal $\hat{\mZ}$~(\cref{fig:intro_mechanism_toy}) in comparison to assigning equal weights to all coarse neighbours.

\section{Experiments}
% To validate the findings and effectiveness of our proposed masked contrastive learning~(MaskCon) framework, in this section, we conduct extensive experiments on various datasets. In \cref{exp:ablations}, we extensively investigate the effect of the hyperparameters.
In this section, extensive experiments are conducted to validate the findings and effectiveness of the proposed Masked Contrastive Learning (MaskCon) framework. The impact of hyperparameters is thoroughly investigated in \cref{exp:ablations}.
In~\cref{exp:cifartoy}, \cref{exp:cifar100} and \cref{exp:in1k}, results are provided for experiments on CIFAR datasets and ImageNet-1K dataset with coarse labels. In~\cref{exp:fine-grained}, we apply our method on more challenging fine-grained datasets, including Stanford Online Products~(SOP) dataset~\cite{sopdataset} and Stanford Cars 196 dataset~\cite{cars196}. 

% To exclude the effect of model capacity and show the robustness of our method, we keep the model settings across all experiments except for our two hyperparameters: $w$ and $\tau$, which will be ablated in~\cref{exp:ablations}.
To ensure the robustness of the proposed method and exclude the impact of model capacity, the model settings are kept consistent throughout all the experiments, except for the two hyperparameters, $w$ and $\tau$, which are ablated in \cref{exp:ablations}.
ResNet18~\cite{resnet} is employed in all experiments, with the initial convolutional layer modified to have a kernel size of 3$\times$3 and a stride of 1, and the initial max-pooling operator removed for CIFAR experiments, to account for the smaller image size~(32$\times$32) as suggested in prior works \cite{moco}. More implementation and dataset details can be found in \cref{appendixa}.
% We use ResNet18~\cite{resnet} in all experiments, specifically, following previous works~\cite{moco}, to account for the smaller image size~(32$\times$32), the initial convolutional layer has been modified to have a kernel size of 3$\times$3 and a stride of 1, and the initial max-pooling operator was removed for CIFAR experiments. Please refer to \textsc{Appendix A} for more implementation and dataset details.

% More results can be found in the \textsc{Appendix A}.
% \paragraph{Implementation details}
% To exclude the effect of model capacity and show the robustness of our method, we keep the model settings across all experiments except for our two hyperparameters: $w$ and $\tau$, which will be ablated in~\cref{exp:ablations}. We use ResNet18~\cite{resnet} in all experiments, specifically, following previous works~\cite{moco}, to account for the smaller image size~(32$\times$32), the initial convolutional layer has been modified to have a kernel size of 3$\times$3 and a stride of 1, and the initial max-pooling operator was removed for CIFAR experiments. We use the SGD optimizer with a momentum of 0.9, and train for 200 epochs with a batch size of 128, a learning rate of 0.02 and cosine annealing lr scheduler, and weight decay as 0.0005. For contrastive learning, we use a memory bank with 8,192 elements, an MLP projector with a hidden dimension of 512, an output dimension of 128, and a temperature for the projection head as $\tau_0=0.1$. For more implementation details, please refer to \textsc{Appendix B}. For more details on the dataset, please refer to \textsc{Appendix C}.

We compare our method with two competing methods: \textbf{Grafit} and \textbf{CoIns}. For a fair comparison, we exhaust the weight $w$ choices for both methods and report the best achievable results in all experiments. Note that when $w=0$, Grafit and CoIns degenerate to self-supervised contrastive learning denoted as \textbf{SelfCon}; Conversely, when $w=1$, Grafit degenerates to supervised contrastive learning~\cite{supcon} denoted as \textbf{SupCon}, while CoIns degenerates to conventional supervised cross-entropy learning denoted as \textbf{SupCE}. 
For reference, we also show the results when training with fine labels -- this is denoted as \textbf{SupFINE}. 
\paragraph{Evaluation protocol}
To evaluate the different methods on the test set with fine labels, we use the recall@K~\cite{oh2016recallretrieval} metric widely used in the image retrieval task. Each test image first retrieves top-K nearest neighbours from the test set and receives 1 if there exists at least one image from the same fine class among the top-K nearest neighbours, otherwise 0. Recall@K averages this score over all the test images.
\begin{figure}[htbp]
    \centering
    \includegraphics[width=0.9\linewidth]{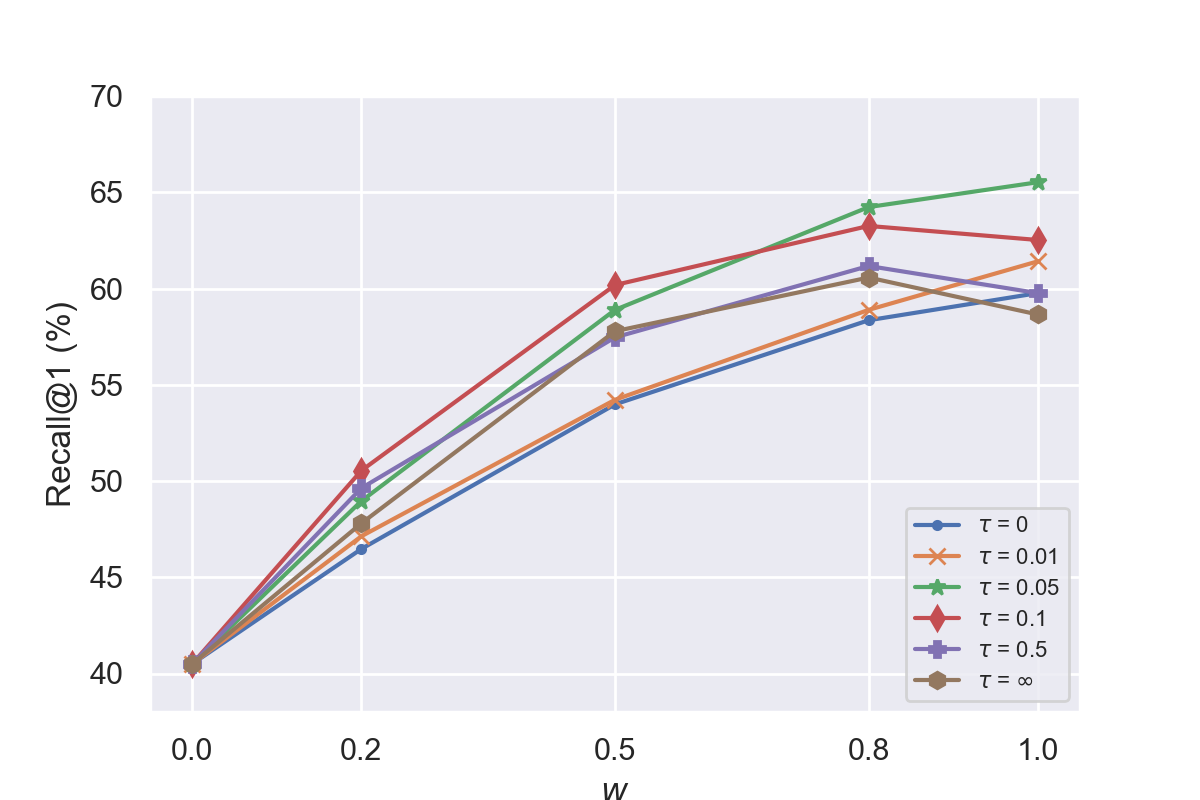}
    \caption{Recall@1 w.r.t $w$ and $\tau$ on CIFAR100 dataset}
    \label{fig:hyperparameters}
\end{figure}

\begin{table*}[htbp]
\centering
\resizebox{0.82\linewidth}{!}{
\begin{tabular}{@{}lcccccccc@{}}
\toprule
\multirow{2}{*}{Method} & \multicolumn{4}{c}{CIFARtoy-goodsplit}                           & \multicolumn{4}{c}{CIFARtoy-badsplit}                             \\ \cmidrule(l){2-5} \cmidrule(l){6-9}
                        & Recall@1       & Recall@2       & Recall@5       & Recall@10      & Recall@1       & Recall@2       & Recall@5       & Recall@10      \\ \midrule
SelfCon                 & 84.83          & 91.55          & 96.35          & 98.16          & 84.83          & 91.55          & 96.35          & 98.16          \\
Grafit                  & 86.61          & 92.33          & 97.01          & 98.38          & 89.96          & 94.36          & 97.61          & 98.10   \\
SupCon                  & 73.84          & 84.25          & 92.14          & 95.46          & 84.66          & 90.93          & 95.15          & 96.71          \\
CoIns                   & 86.15          & 92.76          & 97.21          & 98.46          & 90.55          & 94.94          & 97.73          & 98.71          \\
SupCE                   & 76.30          & 85.26          & 94.65          & 97.46          & 87.15          & 92.85          & 96.78          & 98.34          \\ \midrule
\textcolor{gray}{SupFINE}        & \textcolor{gray}{94.11}          & \textcolor{gray}{96.53}          & \textcolor{gray}{98.25}          & \textcolor{gray}{98.96}   & \textcolor{gray}{94.11}          & \textcolor{gray}{96.53}          & \textcolor{gray}{98.25}          & \textcolor{gray}{98.96}   \\ \midrule
MaskCon~(Ours)          & \textbf{90.28}~\footnotesize{(\textcolor{green}{13.98$\uparrow$})} & \textbf{94.04} & \textbf{97.33} & \textbf{98.53} & \textbf{91.56}~\footnotesize{(\textcolor{green}{4.41$\uparrow$})} & \textbf{95.23} & \textbf{97.70} & \textbf{98.70}          \\\bottomrule
\end{tabular}
}
\caption{Results on CIFARtoy dataset.}
\label{tab:cifartoy}
\end{table*}

\begin{figure*}[htbp]
    \centering
    \includegraphics[width=0.8\textwidth]{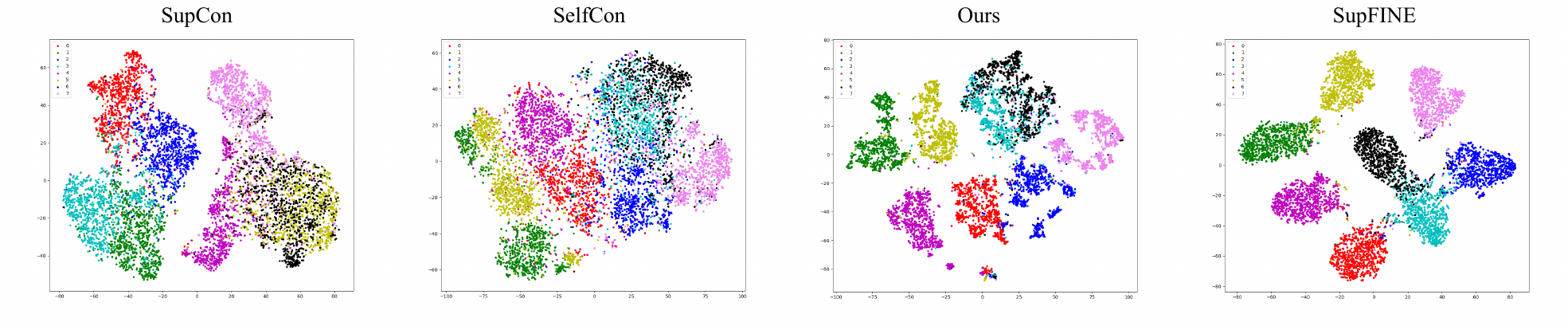}
    \caption{t-SNE visualization of learned representation on CIFARtoy dataset.}
    \label{fig:tsne}
\end{figure*}

\subsection{Effect of $w$ and $\tau$}\label{exp:ablations}
In this section, we extensively ablate the effect of hyperparameters of MaskCon -- weight $w$ and temperature $\tau$. Specifically, we show the recall@1 results on CIFAR100 dataset in \cref{fig:hyperparameters}, with $w = \{0, 0.2, 0.5, 0.8, 1.0\}$ and $\tau = \{0, 0.01, 0.05, 0.1, 0.5, \infty\}$.

The results demonstrate a significant improvement in performance upon the inclusion of additional inter-sample relationships ($w\neq 0$) in comparison to self-supervised learning alone ($w=0$) on the CIFAR dataset. Furthermore, a suitable temperature $\tau$, such as 0.05 or 0.1, consistently yields superior results compared to the simple weighted combination in Grafit ($\tau = \infty$). In general, we recommend initiating the hyperparameter search with $w = 1$ and $\tau = \tau_0$, where $\tau_0$ denotes the temperature for the projection head. Additional results concerning various hyperparameters on other datasets can be found in \cref{appdix b}.
% We find that the performance after introducing additional inter-sample relationships~($w\neq 0$) is significantly better than the results of self-supervised learning alone~($w=0$) on the CIFAR dataset. Moreover, an appropriate temperature $\tau$ such as 0.05 or 0.1 lead to consistent improvement over the simple weighted combination in Grafit~($\tau = \infty$). Generally, we would suggest starting hyperparameter searching from $w = 1$ and $\tau = \tau_0$, where $\tau_0$ is the temperature for the projection head. Please refer to \textsc{Appendix B} for more results with different hyperparameters on other datasets.
%The numbers of fine classes and coarse classes can be used to infer the semantic distance between the coarse labels and the fine labels. As the semantic distance increases, it is reasonable to decrease $w$ and increase $\tau$ accordingly. 

\subsection{Experiments on CIFARtoy dataset} \label{exp:cifartoy}
To simulate the coarse labelling process, we manually generate toy datasets based on the CIFAR10 dataset. Following the problem motivation, a subset of 8 classes is selected from the original 10 classes. Specifically, the classes 'airplane', 'truck', 'automobile', and 'ship' are designated as Class A, while 'horse', 'dog', 'bird', and 'cat' are assigned to Class B. Note that Class A comprises non-organic objects, while Class B comprises animals. Moreover, a bad split of the aforementioned dataset is defined, wherein 'airplane', 'automobile', 'bird', and 'cat' are Class A, and 'horse', 'dog', 'ship', and 'truck' are Class B.
% Following the problem motivation, we select 8 classes out of the 10 original classes: 'airplane', 'truck', 'automobile' and 'ship' as Class A, while 'horse', 'dog', 'bird' and 'cat' as class B. Here, Class A can be easily identified visually as a non-organic object while class B as animals. Moreover, for comparison we define a bad split on the above dataset: 'airplane', 'automobile', 'bird' and 'cat' as class A, while 'horse', 'dog', 'ship' and 'truck' as class B. 

% \begin{mdframed}[linecolor=lightgray!20,backgroundcolor=lightgray!20,everyline=true,linewidth=2pt,roundcorner=20pt]%%<---addendum
As shown in \cref{tab:cifartoy}, our method achieves the best performance in both splits. Moreover, the improvement over the supervised learning with coarse labels~(SupCE) on good split~(13.98\%) is much higher than the bad split~(4.41\%), which shows the substantial potential of MaskCon in dealing with more realistic coarse-labelled datasets~\footnote{
% It may be counter-intuitive to have better retrieval performance on the bad split than on the good split at first sight. 
%The good split can be simply generated based on the visual semantics, while it would require the annotators to be more knowledgeable to generate such a bad split. 
It is worth noting that the better performance on the bad split may seem counter-intuitive at first glance.
However, note that a bad split actually makes the labels more informative, as it explicitly helps to discriminate visually similar classes, e.g., 'cat' and 'dog' in different coarse classes in our setting.}.

In \cref{fig:tsne} we visualize the learned \textit{features} of all test samples in the good split. Clearly, the learned representations with supervised contrastive learning~(SupCon) and self-supervised contrastive learning~(SelfCon) tend to under-cluster and over-cluster the samples, respectively. By contrast, our method gets more compact and clear clusters, in line with the results when trained with fine labels~(SupFINE).
% \end{mdframed}

% \vspace{-10mm}
\subsection{Experiments on CIFAR100 dataset}\label{exp:cifar100}
The common CIFAR100 dataset has 20 classes of coarse labels in addition to the 100 classes of fine labels, with each coarse class containing five fine-grained classes~(500 samples). The results in \cref{tab:cifar100} show that our method achieves significant improvements over the SOTAs. In particular, it improves the top-1 retrieval precision from 47.25\% to 65.52\%, approaching the results by the model learned with fine labels~(71.13\%). 
\begin{table}[htbp]
\centering
\resizebox{0.95\linewidth}{!}{
\begin{tabular}{@{}lcccc@{}}
\toprule
Method         & Recall@1       & Recall@2       & Recall@5       & Recall@10      \\ \midrule
SelfCon        & 40.50          & 51.83          & 66.23          & 76.66          \\
Grafit         & 60.57          & 71.13          & 82.32          & 89.21          \\
SupCon         & 58.65          & 70.04          & 82.18          & 89.09          \\
CoIns          & 60.10          & 70.89          & 83.14          & \textbf{89.52} \\
SupCE          & 47.25          & 61.24          & 77.78          & 87.01          \\ \midrule
\textcolor{gray}{SupFINE}        & \textcolor{gray}{71.13}          & \textcolor{gray}{80.03}  & \textcolor{gray}{87.61}          & \textcolor{gray}{91.59}          \\ \midrule
MaskCon~(Ours) & \textbf{65.52}~\footnotesize{(\textcolor{green}{18.17$\uparrow$})} & \textbf{74.46} & \textbf{83.64} & 89.25          \\ \bottomrule
\end{tabular}
}
\caption{Results on CIFAR100 dataset.}
\label{tab:cifar100}
\end{table}

% Besides, for

\subsection{Experiments on ImageNet-1K dataset}\label{exp:in1k}
In this section, we evaluate our method on the large-scale ImageNet-1K dataset. For efficiency, we experiment with the downsampled version of ImageNet-1K dataset~\cite{chrabaszcz2017downsampled}, where each sample was resized to 32$\times$32. 
Since no official coarse labels exist for ImageNet, we introduce coarse labels based on the WordNet~\cite{miller1998wordnet} hierarchy so as to artificially group the whole dataset into 12 coarse classes: `0:~Invertebrate', `1:~Domestic animal', `2:~Bird', `3:~Mammal', `4:~Reptile/Aquatic vertebrate', `5:~Device', `6:~Vehicle', `7:~Container', `8:~Instrument', `9:~Artifact', `10:~Clothing' and `11:~Others'. In \cref{tab:imagenet}, we show the results on downsampled ImageNet-1K datasets. 

\begin{table}[htbp]
\centering
\resizebox{0.95\linewidth}{!}{
\begin{tabular}{@{}lcccc@{}}
\toprule
Method                    & Recall@1                & Recall@2                & Recall@5                & Recall@10               \\ \midrule
SelfCon                   & 10.28                   & 14.15                   & 22.36                   & 30.34                   \\
Grafit                    & 18.13                   & 25.46                   & 37.19                   & 46.64                   \\
SupCon                    & 13.36                   & 19.40                   & 29.77                   & 39.38                   \\
CoIns                     & 18.36                   & 25.54                   & 37.09                   & 46.89                   \\ 
SupCE                     & 12.23                   & 14.15                   & 27.76                   & 37.03                   \\
\midrule
\textcolor{gray}{SupFINE} & \textcolor{gray}{33.97} & \textcolor{gray}{44.55} & \textcolor{gray}{57.23} & \textcolor{gray}{65.77} \\ \midrule
MaskCon~(Ours)                   & \textbf{19.08}~\footnotesize{(\textcolor{green}{6.86$\uparrow$})}          & \textbf{26.21}          & \textbf{38.17}          & \textbf{47.96}          \\ \bottomrule
\end{tabular}
% \begin{tabular}{@{}ccccc@{}}
% \toprule
% Grafit~($w=0.8$) & Grafit~($w=1$) & Grafit~($w=0$) & \textcolor{gray}{SupFINE} & MaskCon~($w=0.8$) \\ \midrule
%  18.36      &       &       &   \textcolor{gray}{36.38}      &                \\ \bottomrule
% \end{tabular}
}
\caption{Results on ImageNet-1K dataset.}
\label{tab:imagenet}
\end{table}

We can again validate that with temperature-controlled soft relations, our method surpasses SupCE, Grafit and CoIns consistently, especially for top-10 retrieval results.

\subsection{Experiments on fine-grained datasets}\label{exp:fine-grained}
In this section, we conduct experiments in a more challenging scenario -- fine-grained datasets with only coarse labels. Please note, that in this work, we do not aim to compare with the state-of-the-art works in fine-grained classification, which usually involve more specialized techniques, such as object localization and local feature extraction.

\subsubsection{Stanford Online Products~(SOP)}
The Stanford Online Products~(SOP) dataset~\cite{sopdataset} consists of 22,634 products, with each product having between two and twelve photos from different perspectives, for a total of 120,053 images. In addition, there are 12 coarse classes based on the semantic categories of the products, such as 'bicycle' and 'kettle'. For the common image retrieval task, the category of the test set is possibly unknown, so we firstly extract a subset with an unknown test category based on the SOP dataset. We then selected the categories with eight or more images. We then split almost equally this subset to obtain the training set and the test set with 25,368 and 25,278 images, consisting of 2,517 and 2,518 classes respectively, denoted as SOP-split1. Moreover, we select all products with twelve images, and then randomly select ten train images and two test images. Thus, we have 1,498 classes with a total of 17,976 images~(2,996 test images and 14,980 train images), denoted as SOP-split2. 

\begin{table}[htbp]
\centering
\resizebox{0.95\linewidth}{!}{
\begin{tabular}{@{}lcccc@{}}
\toprule
Method         & Recall@1       & Recall@2       & Recall@5       & Recall@10      \\ \midrule
SelfCon        & 70.36          & 75.57          & 81.53          & 85.13          \\
Grafit         & 74.02          & 78.82          & 84.13          & 87.91          \\
SupCon         & 53.69          & 59.55          & 67.12          & 72.78          \\
CoIns          & 70.84          & 76.01          & 82.2           & 86.08          \\
SupCE          & 36.35          & 42.39          & 50.30          & 56.52          \\ \midrule
\textcolor{gray}{SupFINE}        & \textcolor{gray}{83.94}          & \textcolor{gray}{88.04}          & \textcolor{gray}{91.95}          & \textcolor{gray}{94.00}          \\ \midrule
MaskCon~(Ours) & \textbf{74.05}~\footnotesize{(\textcolor{green}{37.7$\uparrow$})} & \textbf{78.97} & \textbf{84.48} & \textbf{87.96} \\ \bottomrule
\end{tabular}
}
\caption{Results on SOP-split1 dataset.}
\label{tab:sopsplit1}
\end{table}

%Moreover, we select all products with twelve images, and then randomly select ten train images and two test images. Thus, we have 1,498 classes with a total of 17,976 images~(2,996 test images and 14,980 train images), denoted as SOP-split2. 

\begin{table}[htbp]
\centering
\resizebox{0.95\linewidth}{!}{
\begin{tabular}{@{}lcccc@{}}
\toprule
Method         & Recall@1       & Recall@2       & Recall@5       & Recall@10      \\ \midrule
SelfCon        & 35.85          & 41.46          & 49.77          & 56.11          \\
Grafit         & 39.12          & 44.66          & 53.10          & 59.65          \\
SupCon         & 25.07          & 29.24          & 35.85          & 41.59          \\
CoIns          & 38.22          & 45.19          & 54.37          & 61.18          \\
SupCE          & 22.56          & 26.34          & 33.28          & 38.95          \\ \midrule
\textcolor{gray}{SupFINE}        & \textcolor{gray}{69.56}          & \textcolor{gray}{75.70}          & \textcolor{gray}{83.24}          & \textcolor{gray}{87.88}          \\ \midrule
MaskCon~(Ours) & \textbf{45.36}~\footnotesize{(\textcolor{green}{22.8$\uparrow$})} & \textbf{51.07} & \textbf{58.91} & \textbf{65.52} \\ \bottomrule
\end{tabular}
}
\caption{Results on SOP-split2 dataset.}
\label{tab:sopsplit2}
\end{table}

\begin{figure*}[htbp]
    \centering
    \includegraphics[width=0.86\textwidth]{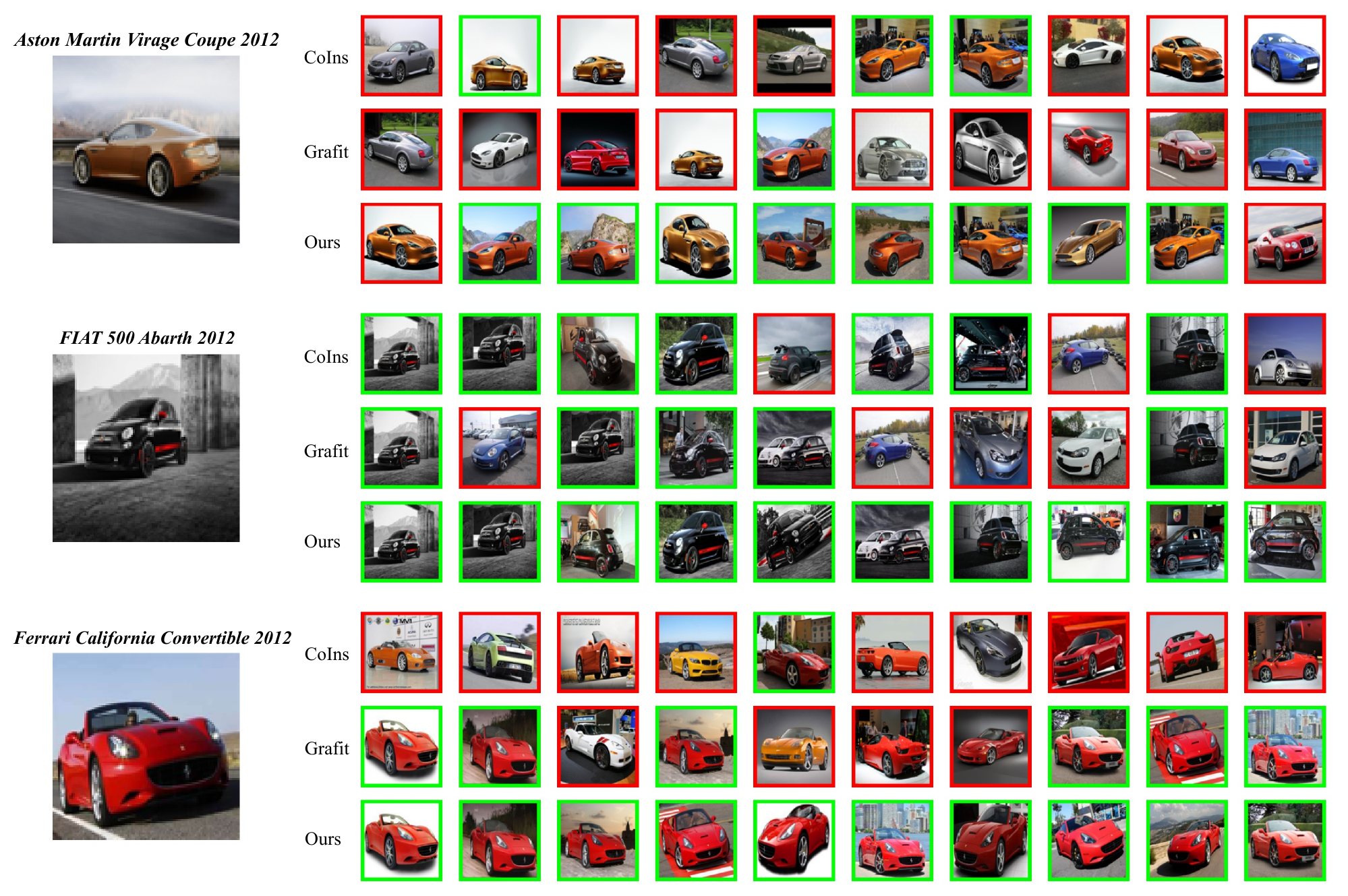}
    \caption{Examples of Top-10 image retrieval results on Cars196 dataset.}
    \label{fig:cars_retrieval}
\end{figure*}
% \vspace{-2mm}

In \cref{tab:sopsplit1} and \cref{tab:sopsplit2} we show results on SOP-split1 and SOP-split2, respectively. Similarly, regardless of whether the test class is known or not, our method still achieves significant improvements over state-of-the-art methods. 
%We would like to recall \cref{exp:ablations}. 
An interesting phenomenon here is, that on the SOP dataset, SelfCon alone is significantly better than its SupCE/SupCon counterpart. We conjecture that this is due to the fact that each fine class of the SOP dataset in fact consists of different views of the same product, and the number of images per fine class is quite small~(2 to 12). For such a dataset, the classification of fine classes is closer to instance discrimination rather than the coarse classes classification.

\subsubsection{Stanford Cars196}
Stanford Cars196 is another widely used fine-grained image classification benchmark. As there are no official annotated coarse labels, we manually group 196 car models into 8 coarse classes based on the common types of cars: `0:~Cab', `1:~Sedan', `2:~SUV', `3:~Convertible', `4:~Coupe', `5:~Hatchback', `6:~Wagon' and `7:~Van'.
% \vspace{-2mm}

In \cref{tab:cars196}, our method again has the best performance. In addition, unlike in the SOP datasets, SupCon and SupCE have an overwhelming advantage over SelfCon. In this dataset, this is not surprising because for each fine class, there are more images~(ranging from 24 images to 68 images) with different colours/backgrounds in the cars dataset, leading to a much higher intra-class variance. Some image retrieval examples can be found in ~\cref{fig:cars_retrieval}.
\begin{table}[htbp]
\centering
\resizebox{0.95\linewidth}{!}{
\begin{tabular}{@{}lcccc@{}}
\toprule
Method         & Recall@1       & Recall@2       & Recall@5       & Recall@10      \\ \midrule
SelfCon        & 20.97          & 28.75          & 41.44          & 52.61          \\
Grafit         & 42.30          & 54.79          & 71.1           & 81.74          \\
SupCon         & 42.30          & 54.79          & 71.1           & 81.74          \\
CoIns          & 42.77          & 55.60          & 72.29          & 82.53          \\
SupCE          & 42.77          & 55.60          & 72.29          & 82.53          \\ \midrule
\textcolor{gray}{SupFINE}        & \textcolor{gray}{78.09}          & \textcolor{gray}{82.97}          & \textcolor{gray}{86.51}          & \textcolor{gray}{88.04}          \\ \midrule
MaskCon~(Ours) & \textbf{45.53}~\footnotesize{(\textcolor{green}{2.76$\uparrow$})} & \textbf{58.56} & \textbf{74.36} & \textbf{84.36} \\ \bottomrule
\end{tabular}
}
\caption{Results on Stanford Cars196 dataset.}
\label{tab:cars196}
\end{table}
\vspace{-5mm}

\subsubsection{Additional discussion}
Compared to the state-of-the-art, our method achieves the best results in all experiments. We would like to note here some of our insights from the results on the CIFAR and the fine-grained ones~(including ImageNet-1K). More specifically, we note that on the CIFAR datasets the performance of our method approaches that of the supervised one with fine labels~(SupFINE), while on the fine-grained datasets there is a larger gap.
%that in comparison to the results obtained on CIFAR datasets there are 
%However, compared to the great improvements in the above CIFAR datasets, the gap between our method and the supervised counterpart ~(SupFINE) is much larger. We would like to raise further discussion here, as the fine-grained dataset may be of higher interest in the targeted problem setting. 
We note that the critical question is to consider whether the instance discrimination task is useful for fine-grained classification in different scenarios, since coarse labelling usually does not lead to worse results.
%However, compared to the great improvements in the above CIFAR datasets, the gap between our method and the supervised counterpart ~(SupFINE) is much larger. We would like to raise further discussion here, as the fine-grained dataset may be of higher interest in the targeted problem setting. The core is to consider whether the instance discrimination task is useful for fine-grained classification in different data scenarios, since coarse labelling usually does not lead to worse results.
Towards that, recent works have shown that the core idea of self-supervised contrastive learning --- augmentation invariance may be destructive for fine-grained tasks~\cite{xiao2020whatdoesnt_finegrained1, cole2022whendoes_finegrained2}. For example, the commonly applied colour distortion augmentation will promote the model to be non-sensitive to colour information. However, colour may be the key to discriminate between different breeds of birds. How to adapt the instance discrimination task to the fine-grained coarse-labelled dataset is a future direction we wish to pursue.

% \subsection{Results on ImageNet-1K dataset}
% Following \cite{xu2021coins}, we also conduct experiments on ImageNet-1K dataset. To generate coarse labels, we utlize the wordnet blablabla. In \cref{tab:imagenet}

% \begin{table}[htbp]
% \centering
% \resizebox{1\linewidth}{!}{
% \begin{tabular}{@{}lcccc@{}}
% \toprule
% Method         & Recall@1       & Recall@2       & Recall@4       & Recall@8      \\ \midrule
% MoCo           & 42.8          & 52.9          & 62.4          & 71.1          \\
% CoIns          & 51.4          & 61.5          & 70.7          & 78.43          \\
% Grafit*        &           &           &           &           \\ \midrule
% \textcolor{gray}{SupFINE}        & \textcolor{gray}{66.4}          & \textcolor{gray}{75.3}          & \textcolor{gray}{82.1}          & \textcolor{gray}{87.3}          \\ \midrule
% MaskCon~(Ours) & \textbf{} & \textbf{} & \textbf{} & \textbf{} \\ \bottomrule
% \end{tabular}
% }
% \caption{Results on ImageNet-1K dataset. The results of MoCo and SupFINE are copied from \cite{moco}. Results of CoIns are copied from \cite{xu2021coins}. We reproduced the results of Grafit as there is no reported number in the original paper.}
% \label{tab:imagenet}
% \end{table}

\section{Conclusion}
In this work, we propose a \textbf{Mask}ed \textbf{Con}trastive learning framework~(\textbf{MaskCon}) for learning fine-grained information with coarse-labelled datasets. On the basis of two baseline methods, we utilize coarse labels and the instance discrimination task to better estimate inter-sample relations. We show theoretically that our method can reduce the optimization error bound. Extensive experiments with various hyperparameter settings on multiple benchmarks, including the CIFAR datasets and the more challenging fine-grained classification datasets  show that our method achieves consistent and large improvement over the baselines. 

{\setlength{\parindent}{0.0cm}
\textbf{Acknowledgments:} This work was supported by the EU H2020 AI4Media No. 951911 project.
} 
%The next step will be to further improve the performance on fine-grained datasets, especially about the necessary adaptations of the instance discrimination task.

%%%%%%%%% REFERENCES
{\small
\bibliographystyle{ieee_fullname}
\bibliography{egbib}
}
\onecolumn
\appendix

\section{Implementation and dataset details}\label{appendixa}
\subsection{Implementation details}

\paragraph{Hyperparameter settings}
For datasets except ImageNet-1K, we use the SGD optimizer with a momentum of 0.9, and train for 200 epochs with a batch size of 128, a learning rate of 0.02 and cosine annealing lr scheduler, and weight decay as 0.0005. For ImageNet-1K dataset, we train for 100 epochs with a batch size of 256. For contrastive learning, we use a memory bank with 8,192 elements, an MLP projector with a hidden dimension of 512, an output dimension of 128, and a temperature for the projection head as $\tau_0=0.1$. Moreover, we show the MaskCon-specific hyperparameters $\tau$ and $w$ corresponding to the reported numbers in the main paper below~(\cref{tab:hyperparameters}). This can be further validated in \cref{appendixb}.
\begin{table}[htbp]
\centering
\resizebox{0.4\linewidth}{!}{
\begin{tabular}{@{}l|cc|cc|cc@{}}
\toprule
\multirow{2}{*}{Dataset} & \multicolumn{2}{c|}{SOP-split1} & \multicolumn{2}{c|}{SOP-split2} & \multicolumn{2}{c}{Cars196} \\ \cmidrule(l){2-7} 
                         & $w$           & $\tau$          & $w$           & $\tau$          & $w$         & $\tau$        \\ \midrule
Grafit                   & 0.8           & /               & 0.8           & /               & 1           & /             \\
CoIns                    & 0.5           & /               & 0.2           & /               & 1           & /             \\
Ours                     & 0.8           & 0.1             & 0.5           & 0.05            & 1           & 0.1           \\ \bottomrule
\end{tabular}}
\caption{Hyperparameters in the main paper.}
\label{tab:hyperparameters}
\end{table}

\paragraph{Augmentation strategies}
The augmentation strategy is critical for contrastive learning, especially for self-supervised contrastive learning. Following ~\cite{ascl, zheng2021ressl}, we apply strong and weak augmentations for the query and key view respectively. As discussed in the main paper, the strong augmentation may be destructive for fine-grained image classification. In particular, considering that each class in the SOP dataset consists of different views of the same product, we, therefore, use weak image augmentations for both query and key view. We list the applied data augmentations for each dataset as below in the \textsc{pytorch} format, the complete repo will be released upon acceptance.

\noindent\textit{\textbf{Strong augmentation}}
\begin{lstlisting}[language=Python]
if dataset == 'cars196':
    strong_transform = transforms.Compose([
    transforms.RandomResizedCrop(size, (0.2, 1)),
    transforms.RandomHorizontalFlip(),
    transforms.RandomPerspective(0.5, 0.5),
    transforms.RandomApply([transforms.ColorJitter(0.4, 0.4, 0.4, 0.1)], p=0.8),
    transforms.RandomGrayscale(p=0.2),
    transforms.RandomApply(
    [GaussianBlur([.1, 2.])], p=0.5),
    transforms.ToTensor(),
    normalize,
    ])
elif dataset == 'sop_split1' or dataset == 'sop_split2':
    strong_transform = transforms.Compose([
    transforms.RandomResizedCrop(size, (0.2, 1)),
    transforms.RandomHorizontalFlip(),
    transforms.RandomPerspective(0.5, 0.5),
    transforms.ToTensor(),
    normalize,
    ])
else:
    strong_transform = transforms.Compose([
    transforms.RandomResizedCrop(size, (0.2, 1)),
    transforms.RandomHorizontalFlip(),
    CIFAR10Policy(),
    transforms.ToTensor(),
    normalize,
    ])
\end{lstlisting}
% , caption=Weak augmentation]
\noindent\textit{\textbf{Weak augmentation}}
\begin{lstlisting}[language=Python]
if dataset == 'cars196' or dataset == 'sop_split1' or dataset == 'sop_split2':
    weak_transform = transforms.Compose([
    transforms.RandomResizedCrop(size, (0.2, 1)),
    transforms.RandomPerspective(0.5, 0.5),
    transforms.RandomHorizontalFlip(),
    transforms.ToTensor(),
    normalize
    ])
else:
    weak_transform = transforms.Compose([
    transforms.RandomResizedCrop(size, (0.2, 1)),
    transforms.RandomHorizontalFlip(),
    transforms.ToTensor(),
    normalize,
    ])
\end{lstlisting}

% \vskip 0.5cm
\noindent\textit{\textbf{Test augmentation}}
\begin{lstlisting}[language=Python]
test_transform = transforms.Compose([
    transforms.Resize(size),
    transforms.ToTensor(),
    normalize,
    ])
\end{lstlisting}

\subsection{Dataset details}
\paragraph{Stanford Online Products~(SOP)}
For Stanford Online Products~(SOP) datasets, we artificially generate two different splits. There are very limited images for each product in the original SOP datasets, ranging from 2 to 12~(\cref{fig:sop_statistics}). Considering an extreme case, that is, for each product there is only one view image, which will lead us to self-supervised learning based on instance discrimination. To avoid this, we thus select those classes with more images. More specifically, for SOP-split1, we select the classes with eight or more images and obtain a subset consisting of 5,035 classes. We evenly split~(2,517 train classes and 2,518 test classes) this subset to obtain the training set and the test set with 25,368 and 25,278 images, respectively. For SOP-split2, we select all products with twelve images, and then randomly select ten train images and two test images. Thus, we have 1,498 classes with a total of 17,976 images~(2,996 test images and 14,980 train images), respectively.
\begin{figure}[htbp]
    \centering
    \includegraphics[width=0.4\linewidth]{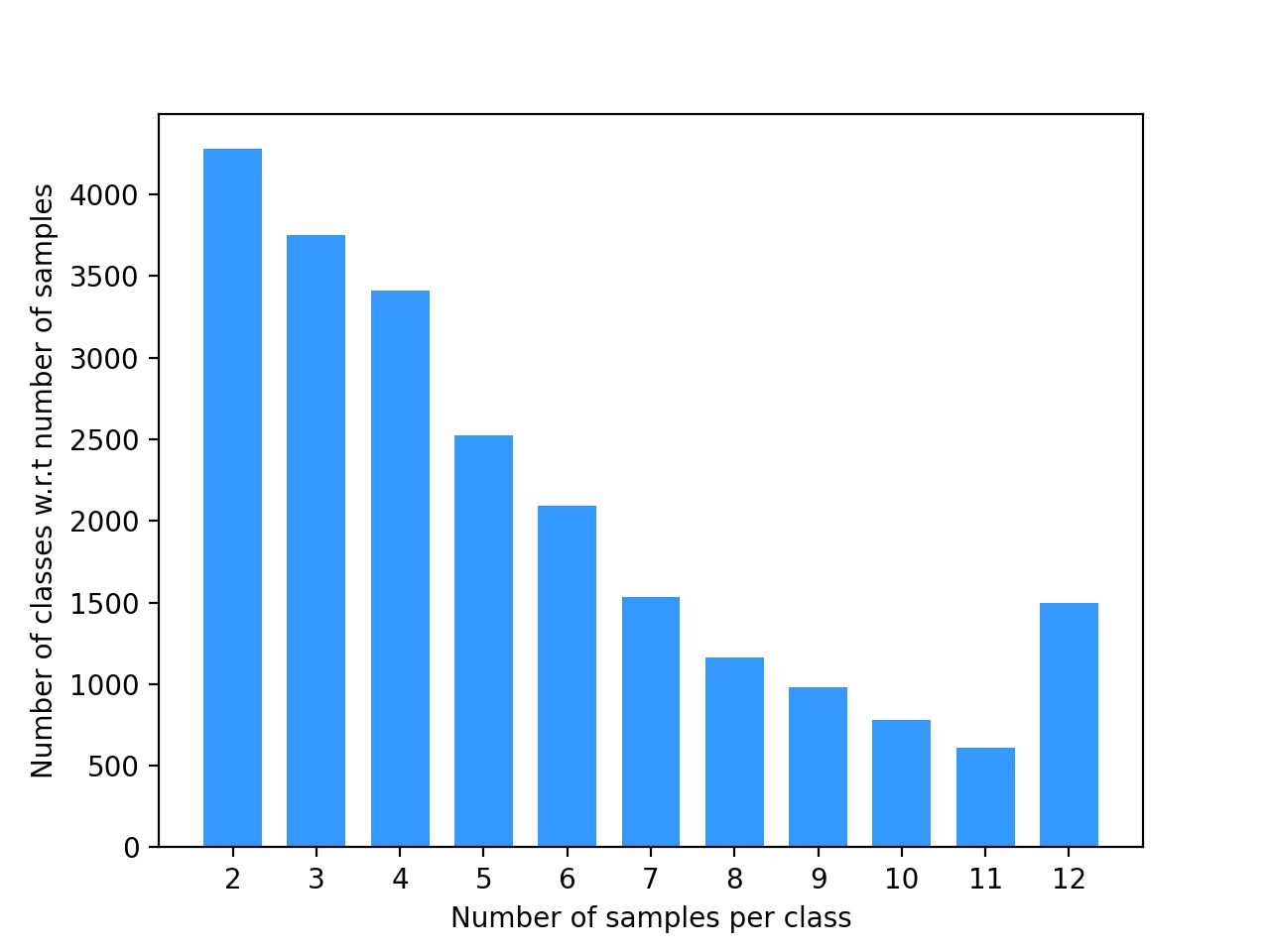}
    \caption{Statistics on SOP dataset.}
    \label{fig:sop_statistics}
\end{figure}

% For each split, we provide the fine and coarse labels and corresponding image paths for train and test set, e.g., \textit{sop\_split1/train\_fine\_label.json} as the fine labels for SOP-split1.

\begin{figure*}[htbp]
    \centering
    \includegraphics[width=\textwidth]{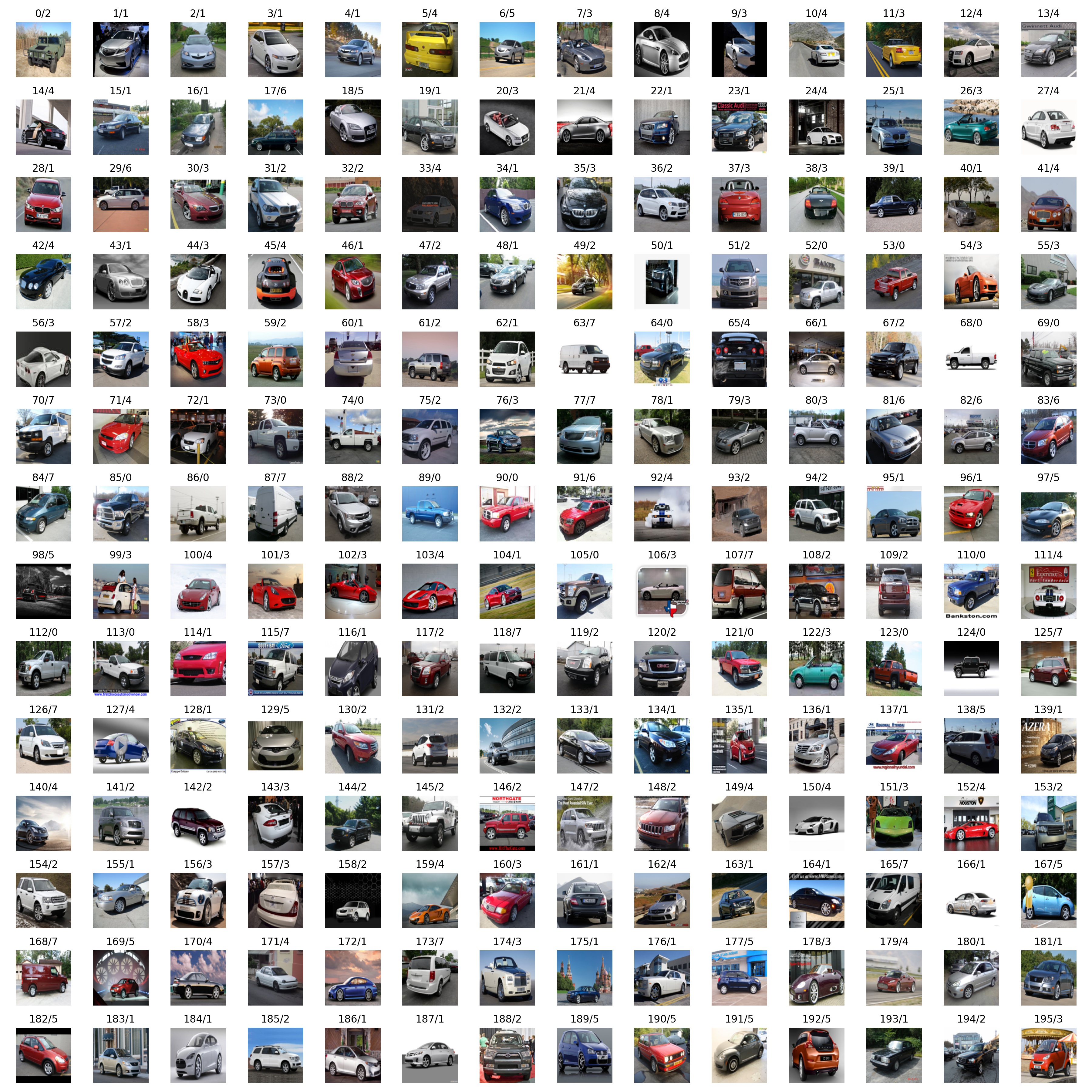}
    \caption{Coarse labels for each of the fine classes---each image captioned as \{fine label\}/\{coarse label\}.}
    \label{fig:coarse_cars196}
\end{figure*}
\paragraph{Stanford Cars196}
For the Cars196 dataset, similarly, there are no official coarse labels. For this reason, we manually group 196 car models into 8 coarse classes based on the common types of cars: `0:~Cab', `1:~Sedan', `2:~SUV', `3:~Convertible', `4:~Coupe', `5:~Hatchback', `6:~Wagon' and `7:~Van'.
\begin{table}[htbp]
\centering
\begin{tabular}{@{}cccccccc@{}}
\toprule
0  & 1  & 2  & 3  & 4  & 5  & 6 & 7  \\ \midrule
18 & 50 & 36 & 30 & 29 & 14 & 7 & 12 \\ \bottomrule
\end{tabular}
\caption{Number of fine classes in each coarse class on Cars196.}
\label{tab:carssplit}
\end{table}
In \cref{fig:coarse_cars196}, we provide some example images from original fine classes to coarse classes for the Stanford Cars196 dataset. 
% The detailed class mapping file can be found in \textit{cars\_mapping.json} and also in the code repository.

\paragraph{ImageNet-1K}
In this section, we evaluate our method on the large-scale ImageNet-1K dataset. For efficiency, we experiment with the downsampled version of ImageNet-1K dataset~\cite{chrabaszcz2017downsampled}, with each sample resized to 32$\times$32. 
\begin{table}[htbp]
\centering
% \resizebox{1\linewidth}{!}{
\begin{tabular}{@{}cccccccccccc@{}}
\toprule
0  & 1   & 2  & 3  & 4  & 5   & 6  & 7  & 8   & 9  & 10 & 11 \\ \midrule
61 & 123 & 59 & 95 & 60 & 130 & 70 & 53 & 105 & 71 & 93 & 80 \\ \bottomrule
\end{tabular}
% }
\caption{Number of fine classes in each coarse class.}
\label{tab:imagenetsplit}
\end{table}
Since there are no officially coarse labels for imagenet, we here introduce coarse labels based on the WordNet~\cite{miller1998wordnet} hierarchy so as to artificially group the whole dataset into 12 coarse classes: `0:~Invertebrate', `1:~Domestic animal', `2:~Bird', `3:~Mammal', `4:~Reptile/Aquatic vertebrate', `5:~Device', `6:~Vehicle', `7:~Container', `8:~Instrument', `9:~Artifact', `10:~Clothing' and `11:~Others'. 
% The detailed mapping from original fine classes to coarse classes can be found in the accompanying file \textit{imagenet\_fine\_to\_coarse.json} in the supplementary material. 

\paragraph{Class imbalance}
We note that the number of fine classes in each coarse class is highly imbalanced~(\cref{tab:imagenetsplit}, \cref{tab:carssplit}). In this work, we do not explicitly deal with it~(for e.g., dataset resampling), as we believe, that such class imbalance can be common in the targeted problem setting, and we aim to test our method in such more realistic setting. 

\section{Additional experiments and results}\label{appendixb}

As mentioned in the main content, we perform a search for the hyperparameters for both Grafit and CoIns and report the best results. Here, in~\cref{tab:cifar20},~\cref{tab:sop-split1},~\cref{tab:sop-split2},~\cref{tab:cifartoy} and \cref{tab:cars196}, we report the results with so additional different hyperparameter settings.

Please note that we do not investigate fewer values for the hyperparameter $w$ in comparison to Grafit and CoIns. From the tables we can see that with appropriate weights, both Grafit and CoIns get considerable improvements over the supervised and self-supervised only baselines. Our method, achieves even higher improvements with an appropriate $\tau$.
\begin{table}[htbp]
\centering
% \resizebox{0.5\linewidth}{!}{
\begin{tabular}{@{}l|c|c|cccc@{}}
\toprule
Method                  & $w$ & $\tau$ & Recall@1       & Recall@2       & Recall@5       & Recall@10      \\ \midrule
SupFINE                 & /   & /      & 71.13          & 80.03          & 87.61          & 91.59          \\ \midrule
\multirow{4}{*}{Grafit} & 0.2 & /      & 47.80          & 59.41          & 73.29          & 82.17          \\
                        & 0.5 & /      & 57.78          & 68.24          & 80.01          & 87.05          \\
                        & 0.8 & /      & \textbf{60.57} & \textbf{71.13} & \textbf{82.32} & \textbf{89.21} \\
                        & 1.0 & /      & 58.65          & 70.04          & 82.18          & 89.09          \\ \midrule
\multirow{4}{*}{CoIns}  & 0.2 & /      & 56.93          & 68.51          & 80.90          & 87.88          \\
                        & 0.5 & /      & \textbf{60.10} & \textbf{70.89} & \textbf{83.14} & \textbf{89.52} \\
                        & 0.8 & /      & 59.35          & 70.40          & 83.19          & 89.82          \\
                        & 1.0 & /      & 47.25          & 61.24          & 77.78          & 87.00          \\ \midrule
MaskCon                 & 0.5 & 0.01   & 54.22          & 65.31          & 77.45          & 85.12          \\
                        & 0.5 & 0.05   & 58.89          & 68.91          & 80.02          & 86.89          \\
                        & 0.5 & 0.1    & 60.18          & 70.41          & 81.81          & 88.25          \\
                        & 0.5 & 0.5    & 57.46          & 68.27          & 79.86          & 87.05          \\
                        & 1   & 0.01   & 61.41          & 71.22          & 81.44          & 87.70          \\
                        & 1   & 0.05   & \textbf{65.52} & \textbf{74.46} & \textbf{83.64} & \textbf{89.25} \\
                        & 1   & 0.1    & 62.52          & 72.51          & 83.27          & 89.18          \\
                        & 1   & 0.5    & 59.76          & 70.28          & 81.97          & 88.58          \\ \midrule
SelfCon                 & /   & /      & 40.50          & 51.83          & 66.23          & 76.66          \\ \bottomrule
\end{tabular}
% }
\caption{Extra results on CIFAR100 dataset.}
\label{tab:cifar20}
\end{table}

\begin{table}[htbp]
\centering
% \resizebox{0.9\linewidth}{!}{
\begin{tabular}{@{}l|c|c|cccc@{}}
\toprule
Method                  & $w$ & $\tau$ & Recall@1       & Recall@2       & Recall@5       & Recall@10      \\ \midrule
SupFINE                 & /   & /      & 83.94          & 88.04          & 91.95          & 94.00          \\ \midrule
\multirow{4}{*}{Grafit} & 0.2 & /      & 71.04          & 76.17          & 81.85          & 85.51          \\
                        & 0.5 & /      & 72.66          & 77.43          & 82.89          & 86.40          \\
                        & 0.8 & /      & \textbf{74.02} & \textbf{78.82} & \textbf{84.13} & \textbf{87.91} \\
                        & 1.0 & /      & 53.69          & 59.55          & 67.12          & 72.78          \\ \midrule
\multirow{4}{*}{CoIns}  & 0.2 & /      & 70.75          & 76.32          & 82.24          & 86.16          \\
                        & 0.5 & /      & \textbf{70.84} & \textbf{76.01} & \textbf{82.20} & \textbf{86.08} \\
                        & 0.8 & /      & 67.68          & 73.26          & 80.04          & 84.18          \\
                        & 1.0 & /      & 36.35          & 42.39          & 50.30          & 56.52          \\ \midrule
MaskCon                 & 0.2 & 0.1    & 72.35          & 77.34          & 82.92          & 86.46          \\
                        & 0.5 & 0.05   & 70.87          & 76.17          & 82.11          & 85.92          \\
                        & 0.5 & 0.1    & 73.22          & 78.34          & 83.98          & 87.35          \\
                        & 0.8 & 0.05   & 68.53          & 73.72          & 79.86          & 84.22          \\
                        & 0.8 & 0.1    & \textbf{74.05} & \textbf{78.97} & \textbf{84.48} & \textbf{87.96} \\ \midrule
SelfCon                 & /   & /      & 70.36          & 75.57          & 81.53          & 85.13          \\ \bottomrule
\end{tabular}
% }
\caption{Extra results on SOP-split1 dataset.}
\label{tab:sop-split1}
\end{table}

\begin{table*}[htbp]
\centering
% \resizebox{0.8\linewidth}{!}{
\begin{tabular}{@{}l|c|c|cccc|cccc@{}}
\toprule
\multirow{2}{*}{Method}  & \multirow{2}{*}{$w$} & \multirow{2}{*}{$\tau$} & \multicolumn{4}{c|}{CIFARtoy~(good split)}                        & \multicolumn{4}{c}{CIFARtoy~(bad split)}                          \\ \cmidrule(l){4-11} 
                         &                      &                         & Recall@1       & Recall@2       & Recall@5       & Recall@10      & Recall@1       & Recall@2       & Recall@5       & Recall@10      \\ \midrule
SupFINE                  & /                    & /                       & 94.14          & 96.61          & 98.03          & 98.65          & 94.11          & 96.53          & 98.45          & 98.96          \\ \midrule
\multirow{4}{*}{Grafit}  & 0.2                  & /                       & 85.34          & 91.38          & 96.05          & 97.85          & 86.33          & 92.31          & 96.63          & 98.18          \\
                         & 0.5                  & /                       & \textbf{86.61} & \textbf{92.33} & \textbf{97.01} & \textbf{98.38} & 87.94          & 93.63          & 97.59          & 98.76          \\
                         & 0.8                  & /                       & 86.23          & 92.56          & 96.69          & 98.50          & \textbf{89.96} & \textbf{94.36} & \textbf{97.71} & \textbf{98.90} \\
                         & 1.0                  & /                       & 73.84          & 84.25          & 92.14          & 95.46          & 84.66          & 90.93          & 95.15          & 96.71          \\ \midrule
\multirow{4}{*}{CoIns}   & 0.2                  & /                       & 85.74          & 91.88          & 96.41          & 98.60          & 89.80          & 94.13          & 97.18          & 98.30          \\
                         & 0.5                  & /                       & 86.00          & 92.25          & 96.91          & 98.60          & 90.14          & 94.11          & 97.78          & 98.70          \\
                         & 0.8                  & /                       & \textbf{86.15} & \textbf{92.76} & \textbf{97.21} & \textbf{98.76} & \textbf{90.55} & \textbf{94.94} & \textbf{97.73} & \textbf{98.71} \\
                         & 1.0                  & /                       & 76.30          & 8.26           & 94.65          & 97.46          & 87.15          & 92.85          & 96.78          & 98.34          \\ \midrule
\multirow{8}{*}{MaskCon} & 0.5                  & 0.01                    & 87.44          & 93.01          & 96.95          & 98.38          & 88.89          & 93.84          & 97.20          & 98.38          \\
                         & 0.5                  & 0.05                    & 88.09          & 93.18          & 96.99          & 98.51          & 90.21          & 94.41          & 97.41          & 98.56          \\
                         & 0.5                  & 0.1                     & 87.13          & 92.95          & 96.93          & 98.44          & 88.29          & 93.61          & 97.25          & 98.54          \\
                         & 0.5                  & 0.5                     & 86.31          & 92.49          & 96.74          & 98.53          & 88.46          & 93.43          & 97.63          & 98.70          \\
                         & 1                    & 0.01                    & 88.24          & 93.44          & 97.01          & 98.26          & 90.53          & 94.23          & 97.51          & 98.55          \\
                         & 1                    & 0.05                    & \textbf{90.28} & \textbf{94.04} & \textbf{97.33} & \textbf{98.53} & \textbf{91.56} & \textbf{95.23} & \textbf{97.70} & \textbf{98.70} \\
                         & 1                    & 0.1                     & 78.25          & 87.10          & 94.26          & 96.85          & 88.44          & 93.65          & 97.10          & 98.34          \\
                         & 1                    & 0.5                     & 79.58          & 88.10          & 94.58          & 96.99          & 88.25          & 93.33          & 96.99          & 98.19          \\ \midrule
SelfCon                  & /                    & /                       & 84.84          & 91.73          & 96.14          & 97.94          & 84.83          & 91.55          & 96.35          & 98.16          \\ \bottomrule
\end{tabular}
% }
\caption{Extra results on CIFARtoy dataset.}
\label{tab:cifartoy}
\end{table*}

\begin{table}[htbp]
% \resizebox{0.9\linewidth}{!}{
\centering
\begin{tabular}{@{}l|c|c|cccc@{}}
\toprule
Method                  & $w$ & $\tau$ & Recall@1       & Recall@2       & Recall@5       & Recall@10      \\ \midrule
SupFINE                 & /   & /      & 69.56          & 75.70          & 83.24          & 87.88          \\ \midrule
\multirow{4}{*}{Grafit} & 0.2 & /      & 36.85          & 42.52          & 50.23          & 57.04          \\
                        & 0.5 & /      & 37.98          & 43.69          & 51.67          & 57.98          \\
                        & 0.8 & /      & \textbf{39.12} & \textbf{44.66} & \textbf{53.10} & \textbf{59.65} \\
                        & 1.0 & /      & 25.07          & 29.24          & 35.85          & 41.59          \\ \midrule
\multirow{4}{*}{CoIns}  & 0.2 & /      & \textbf{38.22} & \textbf{45.19} & \textbf{54.37} & \textbf{61.18} \\
                        & 0.5 & /      & 38.18          & 45.06          & 54.17          & 61.28          \\
                        & 0.8 & /      & 36.72          & 42.49          & 51.23          & 58.61          \\
                        & 1.0 & /      & 22.56          & 26.34          & 33.28          & 38.95          \\ \midrule
MaskCon                 & 0.2 & 0.05   & 39.69          & 45.69          & 53.70          & 60.05          \\
                        & 0.5 & 0.05   & \textbf{45.36} & \textbf{51.07} & \textbf{58.91} & \textbf{65.52} \\
                        & 0.8 & 0.05   & 43.84          & 50.03          & 58.18          & 64.19          \\
                        & 1.0 & 0.05   & 41.05          & 46.70          & 54.74          & 60.98          \\ \midrule
SelfCon                 & /   & /      & 35.85          & 41.46          & 49.77          & 56.11          \\ \bottomrule
\end{tabular}
% }
\caption{Extra results on SOP-split2 dataset.}
\label{tab:sop-split2}
\end{table}

\begin{table}[htbp]
\centering
% \resizebox{0.9\linewidth}{!}{
\begin{tabular}{@{}l|c|c|cccc@{}}
\toprule
Method                  & $w$ & $\tau$ & Recall@1       & Recall@2       & Recall@5       & Recall@10      \\ \midrule
SupFINE                 & /   & /      & 78.09          & 82.97          & 86.51          & 88.04          \\ \midrule
\multirow{4}{*}{Grafit} & 0.2 & /      & 23.64          & 32.69          & 46.35          & 57.98          \\
                        & 0.5 & /      & 28.83          & 38.96          & 53.99          & 65.08          \\
                        & 0.8 & /      & 35.58          & 47.97          & 64.47          & 75.28          \\
                        & 1.0 & /      & \textbf{42.30} & \textbf{54.79} & \textbf{71.10} & \textbf{81.74} \\ \midrule
\multirow{4}{*}{CoIns}  & 0.2 & /      & 32.30          & 43.19          & 59.11          & 71.10          \\
                        & 0.5 & /      & 36.71          & 48.80          & 65.33          & 76.16          \\
                        & 0.8 & /      & 38.71          & 49.04          & 67.28          & 79.23          \\
                        & 1.0 & /      & \textbf{42.77} & \textbf{55.60} & \textbf{72.29} & \textbf{82.53} \\ \midrule
MaskCon                 & 1   & 0.01   & 27.87          & 38.19          & 52.80          & 64.13          \\
                        & 1   & 0.05   & 33.04          & 43.69          & 59.10          & 70.59          \\
                        & 1   & 0.1    & \textbf{45.53} & \textbf{58.56} & \textbf{74.36} & \textbf{84.36} \\
                        & 1   & 0.5    & 42.84          & 55.23          & 71.57          & 81.72          \\ \midrule
SelfCon                 & /   & /      & 20.97          & 28.75          & 41.44          & 52.61          \\ \bottomrule
\end{tabular}
% }
\caption{Extra results on Stanford Cars196 dataset.}
\label{tab:cars196}
\end{table}

%For a fair comparison, here we report results for Grafit and CoIns with $w = \{0, 0.2, 0.5, 0.8, 1.0\}$. Also, we don't sweep the hyperparameter combinations for our method to chase the best accuracy. 

\end{document}